%% file: main.tex
\documentclass[letterpaper]{article} 
\usepackage{aaai24}  
\usepackage{times}  
\usepackage{helvet}  
\usepackage{courier}  
\usepackage[hyphens]{url}  
\usepackage{graphicx} 
\usepackage{natbib}  
\usepackage{caption} 
\usepackage{algorithm}
\usepackage{algorithmic}

\usepackage{bm}
\usepackage[percent]{overpic}

\usepackage{newfloat}
\usepackage{listings}
\DeclareCaptionStyle{ruled}{labelfont=normalfont,labelsep=colon,strut=off} 
\floatstyle{ruled}
\newfloat{listing}{tb}{lst}{}
\floatname{listing}{Listing}
\title{Social, Legal, Ethical, Empathetic, and Cultural Rules:\\ Compilation and Reasoning (Extended Version)}
\author{
Nicolas Troquard,
Martina {De Sanctis},
Paola Inverardi,\\
Patrizio Pelliccione,
Gian Luca Scoccia
}
\affiliations{
Gran Sasso Science Institute\\
L'Aquila, Italy\\
firstname.surname@gssi.it
}

\input{macros}

\begin{document}

\maketitle

\begin{abstract}
The rise of AI-based and autonomous systems is raising concerns and apprehension due to potential negative repercussions stemming from their behavior or decisions. These systems must be designed to comply with the human contexts in which they will operate. To this extent, \citet{Townsendetal22mindsandmachines} introduce the concept of SLEEC (social, legal, ethical, empathetic, or cultural) rules that aim to facilitate the formulation, verification, and enforcement of the rules AI-based and autonomous systems should obey. They lay out a methodology to elicit them and to let philosophers, lawyers, domain experts, and others to formulate them in natural language. To enable their effective use in AI systems, it is necessary to translate these rules systematically into a formal language that supports automated reasoning. In this study, we first conduct a linguistic analysis of the SLEEC rules pattern, which justifies the translation of SLEEC rules into classical logic. Then we investigate the computational complexity of reasoning about SLEEC rules and show how logical programming frameworks can be employed to implement SLEEC rules in practical scenarios. The result is a readily applicable strategy for implementing AI systems that conform to norms expressed as SLEEC rules.
\end{abstract}

\input{content}

\section*{Ethical Statement}
\new{
The ethical considerations are addressed during the elicitation phase to derive SLEEC rules that are suitable for an application  \cite{Townsendetal22mindsandmachines}. SLEEC rules are designed to prevent the robot's behaviors that could potentially cause harm, thereby avoiding violations of moral principles. Our contribution enables the unambiguous and rigorous implementation of SLEEC rules, contributing to the specification of the actual robot behavior. Nevertheless, a comprehensive ethical review of the entire system remains essential.
}

\section*{Acknowledgements}
The work was partially supported by the PRIN project P2022RSW5W RoboChor: Robot Choreography, the PRIN project 2022JKA4SL HALO: etHical-aware AdjustabLe autOnomous systems, and by the MUR (Italy) Department of Excellence 2023--2027.

\bibliography{bib}

\appendix
\input{supplementarymaterial.tex}

\end{document}

%% file: macros.tex
\usepackage{xspace}
\usepackage{xcolor}
\usepackage{pifont}
\usepackage{amsmath}
\usepackage{amssymb}
\usepackage{ntheorem}
\usepackage{cleveref}

\DeclareMathOperator{\UNLESS}{{\sf UNLESS}\xspace}
\DeclareMathOperator{\INWHICHCASE}{{\sf IN~WHICH~CASE}\xspace}

\newtheorem{definition}{Definition}

\newtheorem{example}{Example}
\newtheorem{proposition}{Proposition}
\newtheorem{fact}{Fact}
\newtheorem*{proof}{Proof}

\newtheorem*{nothing}{\!}

\newcommand{\coNP}{\textsc{coNP}\xspace}
\newcommand{\PTIME}{\textsc{PTIME}\xspace}

\newcommand{\compile}[0]{\textsf{compile}}

\newcommand{\qed}{\hfill $\square$}

\newcommand{\new}[1]{#1}

\newcommand{\afemalename}{Ana\xspace}
\newcommand{\cmark}{\ding{51}}%
\newcommand{\xmark}{\ding{55}}%
\newcommand{\qmark}{{\fontfamily{pnc}\selectfont \textbf{?}}}

\newcommand{\hide}[1]{}

%% file: content.tex
\section{Introduction}
The rise of AI-based and autonomous systems is posing concerns and causing apprehension due to potential negative repercussions arising from their behavior or decisions.
Philosophers, engineers, and legal experts are actively exploring avenues to regulate the actions of these systems, which are becoming increasingly integrated into our daily existence. One way to exert control over these systems involves the identification of rules that dictate their conduct during collaboration and interaction with humans.\footnote{This paper was originally published as \citet{DBLP:conf/aaai/TroquardSIPS24}.}

Normative multi-agent systems~\cite{chopra2018handbook} are systems where the behaviour of the participants is affected by the norms established therein. 
\citet{Townsendetal22mindsandmachines} propose a methodology to elicit normative rules for multi-agent systems. The process follows an iterative approach:
\begin{enumerate}
    \item identify high-level normative principles,
    \item identify robot/system capabilities,
    \item create a preliminary normative rule,
    \item \label{item:identify} identify conflicts between normative rule and robot/system capabilities, and if any:
    \begin{itemize}
    \item resolve normative conflicts,
    \item refine normative rule,
    \item go back to~\ref{item:identify}.
    \end{itemize}
\end{enumerate}
In the process, philosophers, ethicists, lawyers, domain experts, and other professionals can extract normative rules from a specific domain. These rules are referred to as \emph{SLEEC rules}, which stands for \emph{social, legal, ethical, empathetic, and cultural} rules.
These rules are aimed to play a crucial role in facilitating the creation of tools and instruments to engineer systems that conform adequately to the established norms, while also aiding in the verification of existing system behavior. 
However, the SLEEC rules obtained by elicitation are typically expressed in natural language. 
Therefore, to be employed effectively in AI systems, these rules need to be translated into a machine-interpretable language. Fortunately, the elicitation process generates rules that are characterized by a clear pattern.

The primary objective of this paper is to present a systematic approach for translating SLEEC rules into classical logic, thereby facilitating their seamless integration in SLEEC-sensitive AI systems.
%

\begin{example}\label{ex:running}
\citet{Townsendetal22mindsandmachines} describe the rule elicitation methodology on a running example in the domain of nursing homes. At the end of the iterative process, the following SLEEC rule is obtained.
\begin{quote}
When the user tells the robot to open the curtains then the robot should open the curtains, UNLESS the user is `undressed' in which case the robot does not open the curtains and tells the user `the curtains cannot be opened while you, the user, are undressed,' UNLESS the user is `highly distressed' in which case the robot opens the curtains.
\end{quote}
\end{example}
For clarity of exposition and concision, we will follow suit and illustrate the present work with this running example.

\paragraph{Compilation and Reasoning.}
This paper delves into an exploration of the suitability of classical logic and logic programming frameworks for representing and reasoning about \citeauthor{Townsendetal22mindsandmachines}'s rules, emphasizing the simplicity and practicality of these approaches within this context.


In computer science, a \emph{compilation} is a translation of computer code in a language into computer code in another language.
By \emph{normative rule compilation} we intend a \emph{translation of SLEEC rules} from natural language \emph{into a logical language}.

\addtocounter{example}{-1}
\begin{example}[continue]
After 
 disambiguation and identification of the relevant pieces of information, the logical form of the SLEEC rule is ($\lnot\lnot$ is deliberate):
    \[
    (a \land \lnot\lnot d \rightarrow o) \land (a \land  \lnot d \land  \lnot h \rightarrow n \land  s) \land  (a \land  \lnot d \land h \rightarrow o) \enspace .
    \]
We will provide the details of the successive steps in due time.\footnote{For impatient readers, reference \Cref{sec:piecesofinformation} to understand the reading of the propositional variables, and \Cref{ex:compilation}.}
\end{example}
Compiling SLEEC rules into a formal language such as classical logic offers numerous advantages. Firstly, it endows SLEEC rules with a precise semantics, enabling unequivocal determinations of their consistency and whether specific outcomes are necessary consequences of the set of rules. Secondly, it facilitates machine processing of SLEEC rules, allowing the integration of off-the-shelf theorem provers, reasoners, and logic programming frameworks (e.g., PROLOG, Answer Set Programming)\footnote{Robot programming platforms already propose wrapped logic programming frameworks for easier integration, e.g., KnowRob/RosProlog~\cite{rosprolog} and ROSoClingo~\cite{rosoclingo}.} for embedding SLEEC-compliant decision-making modules into AI systems.



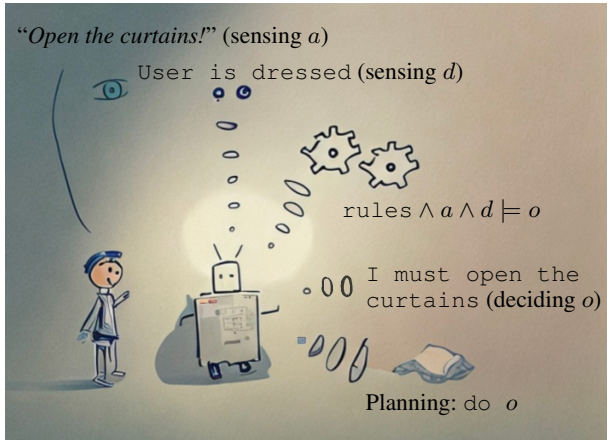
\begin{figure}
\begin{center}
\begin{overpic}[width=0.95\linewidth]{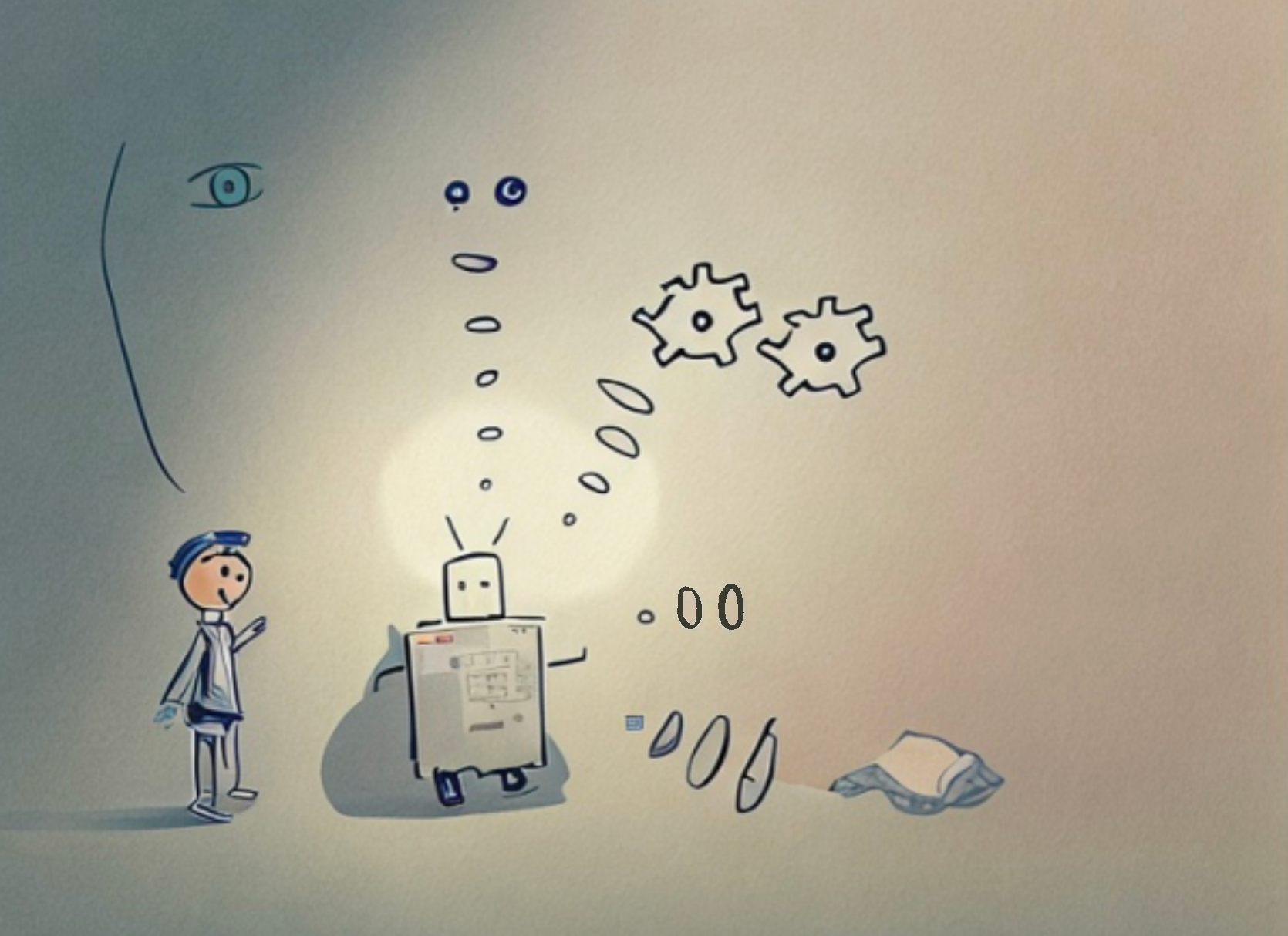}
 \put (2,66) {\small ``\emph{Open the curtains!}'' (sensing $a$)}
 \put (22,60) {\small \texttt{User is dressed} (sensing $d$)}
 \put (56,37) {\small $\texttt{rules} \land a \land d \models o$}
 \put (60,24) {\begin{minipage}{3.2cm}\small
     \texttt{I must open the curtains} (deciding $o$)
 \end{minipage}}
 \put (60,5) {\small Planning: \texttt{do $o$}}
\end{overpic}
\end{center}
\caption{\label{fig:full-picture} SLEEC-compliant AI systems.}
\end{figure}

\addtocounter{example}{-1}
\begin{example}[continue]
\Cref{fig:full-picture} illustrate the central role of reasoning with SLEEC rules. The robot senses the user asking `open the curtains!', setting the propositional variable $a$ to true. The robot senses that the user is dressed, setting the propositional variable $d$ to true. The propositional variable $o$ stands for the obligation to open the curtains.
Automated reasoning is then performed to decide whether the formalised rule (see before), $a$, and $d$ together entail $o$. Since it is case, the robot concludes that it ought to open the curtains. It can now perform some planning task to ensure, e.g., that sometime in the future the curtains are opened.
\end{example}

The literature in multi-agent systems, knowledge representation, and normative reasoning offers numerous ad-hoc formalisms to represent and reason about norms~\cite{10.5555/212501,gabbay2013handbook,gabbay2021handbookvol2}.
These specialized formalisms for normative reasoning hold a great value for representing complex norms \emph{that have not been elicited into SLEEC rules}.
However, we will demonstrate that the specific rules proposed by \citet{Townsendetal22mindsandmachines} can be effectively represented using a simple classical logic approach. Furthermore, implementing these rules in logic programming frameworks such as PROLOG or Answer Set Programming (ASP) is a straightforward process.
Thus, simple formalisms, supported by a wide range of readily-available software support, are sufficient for developing SLEEC-sensitive AI systems. 
\citet{yaman23normativerules} 
have presented SLEECVAL, a tool designed for the validation of SLEEC rules through the mapping of rules to tock-CSP processes. 
Subsequently, a CSP refinement checker can be employed to detect redundant and conflicting rules. \new{A tool to identify conflicts, redundancies, and concerns, aimed at assisting experts during SLEEC rule elicitation, is also presented by \citet{DBLP:conf/kbse/FengMYTCCC23}.}
In contrast, our logic-based compilation is straightforward and allows for versatile reasoning instead of mere verification. We can effortlessly verify the consistency of a set of SLEEC rules and integrate it into a robot development platform to perform decision-making tasks, determining the obligations that arise from a given set of rules and observations as in Figure~\ref{fig:full-picture}.

\paragraph{Outline.}
The reader unfamiliar with propositional logic or logic programming can find some basic definitions in \Cref{sec:logic}. 
In \Cref{sec:unlessinwhichcase}, we conduct a linguistic analysis of the SLEEC rules pattern, which justifies the translation of SLEEC rules into classical logic. In \Cref{sec:pieces} we briefly explain how to take care of ambiguities and to identify the relevant pieces of information in SLEEC rules. In \Cref{sec:compilation}, we define the logical translation of SLEEC rules, thus also establishing a formal semantics. In \Cref{sec:complexity}, we study the computational complexity of reasoning about SLEEC rules.
In \Cref{sec:reasoning}, we demonstrate how logic programming frameworks can be employed to implement SLEEC rules in practical scenarios. 
We conclude in \Cref{sec:conclusion}.

\section{Elements of Logic and Logic Programming}
\label{sec:logic}

To maintain self-contained clarity, we lay out the basic definitions of classical propositional logic, along with brief elements of logic programming. This foundational knowledge will suffice for understanding the forthcoming arguments.

\subsection{Propositional Logic}

Let $AP$ be a non-empty enumerable set of propositional variables.
A \emph{propositional formula} $\phi$ over $AP$ is defined by the following grammar: $\phi ::= p \mid \lnot \phi \mid \phi \lor \phi$,
where $p \in AP$. As it is usual, when $\phi$ and $\psi$ are propositional formulas, we define $\phi \land \psi = \lnot (\lnot \phi \lor \lnot \psi)$, and $\phi \rightarrow \psi = \lnot \phi \lor \psi$, and $\top = q \lor \lnot q$ for some $q \in AP$.

A \emph{positive literal} is a propositional variable $p$ and a \emph{negative literal} is the negation of a propositional variable $\lnot p$.
A \emph{clause} is a disjunction of literals. A \emph{formula in CNF} is a conjunction of clauses. A \emph{formula in 3CNF} is a conjunction of clauses of three literals. 
%


An \emph{interpretation }over $AP$ is a function $v: AP \longrightarrow \{true,false\}$. It can be extended to complex formulas as follows:
\begin{itemize}
    \item $v(\lnot \phi) = true$ iff $v(\phi) = false$.
    \item $v(\phi \lor \psi) = true$ iff $v(\phi) = true$ or $v(\psi) = true$. 
\end{itemize}
A \emph{model} of a formula $\phi$ over $AP$ is an interpretation $v$ over $AP$ such that $v(\phi) = true$.
A propositional formula $\phi$ is \emph{satisfiable} if there is an interpretation $v$ such that $v(\phi) = true$.
Let $\Gamma$ be a set of propositional formulas, and $\phi$ be a propositional formula. We say that $\Gamma$ \emph{entails} $\phi$ if for all interpretations $v$, if $v(\psi) = true$ for all $\psi \in \Gamma$, then $v(\phi) = true$. We write $\Gamma \models \phi$ when it is the case. We write $\phi \equiv \psi$ when the two formulas $\phi$ and $\psi$ are logically equivalent, that is, when $\{\phi\} \models \psi$ and $\{\psi\} \models \phi$.

\subsection{Logic Programming}

\emph{Logic programming} is a programming paradigm using logical formulas to describe a domain.

A program is a list of clauses of the form:
\begin{scriptsize}
\begin{verbatim}
H :- P1, ..., Pn.
\end{verbatim}
\end{scriptsize}
where \texttt{H} is the \emph{head}, and each \texttt{Pi} is part of the \emph{body}.
It can be read as ``if \texttt{P1} and \ldots and \texttt{Pn} then \texttt{H}''.
The head and the parts of the body can be an arbitrary predicate with arbitrary variables. 
E.g.,
\begin{scriptsize}
\begin{verbatim}
motherof(X,Y) :- woman(X), childof(Y,X).
\end{verbatim}
\end{scriptsize}
captures the fact that if \texttt{X} is a woman, and \texttt{Y} is the child of \texttt{X}, then 
\texttt{X} is the mother of \texttt{Y}.
A fact is a clause without a body. E.g.,
\begin{scriptsize}
\begin{verbatim}
woman(jane).
childof(teddy,jane).
\end{verbatim}
\end{scriptsize}
expresses that \texttt{jane} is a woman and \texttt{teddy} is the child of \texttt{jane}.
All predicate names must appear in a rule head (or in a fact).

The comma \texttt{,} represents a conjunction, while the semicolon \texttt{;} represents a disjunction. The negation (as failure) is written \texttt{$\backslash$+} in PROLOG and \texttt{not} in ASP. Finally, a line starting with \texttt{\%} is a comment.

In PROLOG, we can `query' the knowledge base. E.g.,
\begin{scriptsize}
\begin{verbatim}
motherof(M,C).
\end{verbatim}
\end{scriptsize}
will enumerate all couples (\texttt{M}, \texttt{C}) such that \texttt{M} is the mother of \texttt{C}. In our example, it will find exactly one solution: \texttt{M = jane}, \texttt{C = teddy}.

\section{Unless... In Which Case...}
\label{sec:unlessinwhichcase}
We assume that we are given a SLEEC rule, elicited following the methodology of \citet{Townsendetal22mindsandmachines}.
The pivotal logical connective in SLEEC rules is the word `unless'. And importantly, `unless' is typically followed by `in which case'. The aim of this section is to define the meaning of the formulas of the form $\phi \UNLESS \psi \INWHICHCASE \chi$ where $\phi$, $\psi$, and $\chi$ are arbitrary propositional formulas, and to justify the proposed definition.
\new{This justification is crucial
for instilling trust in the compilation process, which is of utmost importance for applications in social contexts.}
We first look at `unless' in isolation, and then provide a logical definition of the construction `unless... in which case...'.

\subsection{`Unless' as XOR}
To persuade both non-logicians and logicians alike of the correct logical interpretation of `unless', it is helpful to take a moment and consider what it is not.
Indeed, it is extremely tempting to interpret $\phi$ UNLESS $\psi$ as an `exclusive or'~\cite{language-proof-logic}. That is, $(\phi \lor \psi) \land \lnot (\phi \land \psi)$, and is equivalent to the biconditional $\lnot \psi \leftrightarrow \phi$. The biconditional $\leftrightarrow$ can also be read `if and only if', and written IFF in English sentences.

For simplicity of exposition, we will step away from SLEEC rules statements for a moment.
Consider the statement~$1a$, and its XOR reading ($1b$).
\begin{itemize}
    \item[$1a$]\label{item:1a}
\afemalename is in Florence UNLESS she is in Rome.
    \item[$1b$]
\afemalename is in Florence IFF she is not in Rome.
\begin{itemize}
    \item[$1b1$] If \afemalename is not in Rome, then \afemalename is in Florence. ($\checkmark$)
    \item[$1b2$] If \afemalename is in Rome, then \afemalename is not in Florence. ($\checkmark$)
\end{itemize}
\end{itemize}
The statement~$1b$ can be decomposed into the statements $1b1$ and $1b2$. Both are 
sensible consequences of~$1b$. But only because of some background knowledge that a person cannot be in two places at the same time. To make the issue clearer, consider the statement $2a$, and its XOR reading ($2b$).
\begin{itemize}
    \item[$2a$]
\afemalename attends the meeting UNLESS she is in Rome.
    \item[$2b$]
\afemalename attends the meeting IFF she is not in Rome.
\begin{itemize}
    \item[$2b1$] If \afemalename is not in Rome, then \afemalename attends the meeting. (\cmark)
    \item[$2b2$] If \afemalename is in Rome, then \afemalename does not attend the meeting. (\qmark) 
\end{itemize}
\end{itemize}
The statement~$2b$ can be decomposed into the statements $2b1$ and $2b2$. The statement $2b1$ is a sensible consequence of~$2b$. But what about the statement $2b2$?
It is in fact \emph{only} a \emph{Gricean conversational implicature}~\cite{grice1975logic}, and not part of the meaning of the original statement. 
To see this, one can apply Grice's \emph{cancellation principle}. It postulates that if a conclusion is not part of the meaning of an assertion, it can be `cancelled' without contradiction by some further elaboration.
E.g.,
\begin{itemize}
    \item[$3a$] 
\afemalename attends the meeting UNLESS she is in Rome, \emph{\emph{but} if she is in Rome she may still attend remotely}.
\begin{itemize}
    \item[$3b1$] If \afemalename is not in Rome, then \afemalename attends the meeting.  (\cmark)
    \item[$3b2$] If \afemalename is in Rome, then \afemalename does not attend the meeting. (\xmark)
\end{itemize}
\end{itemize}
It is clear that the statement $3b1$ is a consequence of the statement $3a$. But the statement $3b2$ is not, because the further elaboration `but if she is in Rome she may still attend remotely' makes it clear that even if \afemalename is in Rome, it is possible that she attends the meeting.

Back to \Cref{ex:running},
`if the user is undressed, then the robot does not open the curtains'
is \emph{only a conversational implicature} and is not part of the meaning of 
`the robot opens the curtains UNLESS the user is undressed'. Indeed, we could add 
`and if the user is undressed, then it depends on whether the user is highly distressed'.
And it is precisely what the full statement of the rule takes care of by adding  ``UNLESS the user is `highly distressed' in which case the robot opens the curtains''.
It takes away the suggestion that `if the user is undressed, then the robot does not open the curtains'.


\subsection{The Logical Form of `Unless'}

The linguistic argument above serves to refute the right-to-left implication of the XOR/IFF formulation of `unless'.
In fact, \citet{Quine59}
contends that $\phi \UNLESS \psi$ must be logically interpreted as $\lnot \psi \rightarrow \phi$. It is exactly the left-to-right implication of the XOR/IFF formulation above. 
We adopt it as a definition.
\begin{definition}\label{def:unless}
$\phi \UNLESS \psi = \lnot \psi \rightarrow \phi$.
\end{definition}
Translating `unless' statements into classical logic becomes a very simple rewriting process.
\begin{example}
The sentence ``the robot should open the curtains, UNLESS the user is `undressed'\phantom{}'' is logically equivalent to the sentence
``IF the user is not `undressed' THEN the robot should open the curtains''.
\end{example}
With a working logical definition of `unless', we effectively establish a semantic condition for it. 
\begin{fact}
    Let $v$ be an interpretation. We have
    \[
    v(\phi \UNLESS \psi) = true \text{ iff } v(\lnot\psi) = false \text{ or }  v(\phi) = true.
    \]
\end{fact}
It comes with certain simple properties, some of which might appear surprising.
\begin{fact}\label{fact:unless}
The following properties hold:
     \begin{enumerate}
    \item equivalence to logical OR: $\phi \UNLESS \psi \equiv \phi \lor \psi$.
     \item associativity: $(\phi \UNLESS \psi) \UNLESS \chi \equiv \phi \UNLESS (\psi \UNLESS \chi)$.
     \item\label{item:commutativity} commutativity: $\phi \UNLESS \psi \equiv \psi \UNLESS \phi$.
     \end{enumerate}
\end{fact}
Despite their surprising nature, these properties of `unless' will not pose a problem, especially in the broader context of SLEEC rules.

\subsection{Unless... in Which Case...}
The SLEEC rules do not use the word `unless' in isolation, but instead in a construction that could be the ternary Boolean expression:
\[\phi \UNLESS \psi \INWHICHCASE \chi \enspace .\]
It is a stronger statement than just ``$\phi \UNLESS \psi$'', specifying not only what should be the case when $\psi$ is not true (i.e., $\phi$), but also what should be the case when $\psi$ is true (i.e., $\chi$).
It can thus be interpreted as the conjunction of $\phi$ UNLESS $\psi$, and IF $\psi$ THEN $\chi$.
Hence, we obtain the following definition:
\begin{definition}\label{def:unlessinwhichcase}
$\phi \UNLESS \psi \INWHICHCASE \chi = (\phi \UNLESS \psi) \land (\psi \rightarrow \chi)$.
\end{definition}
Using \Cref{def:unless}, it is of course equivalent to $(\lnot \psi \rightarrow \phi) \land (\psi \rightarrow \chi)$.
Clearly, we have $\phi \UNLESS \psi \equiv \phi \UNLESS \psi \INWHICHCASE \top$, where $\top$ is the tautology.

This logical definition allows us to establish a semantic condition for it. 
\begin{fact}
    Let $v$ be an interpretation. We have
    \begin{align*}
        v(\phi \UNLESS & ~\psi \INWHICHCASE \chi) = true \text{ iff }\\
        & v(\lnot\psi) = false \text{ or }  v(\phi) = true \text{ and }\\
        & v(\lnot\psi) = true \text{ or }  v(\chi) = true \enspace .
    \end{align*}
\end{fact}

The following example will aid in clarifying the intuitions.
\begin{example}
    The sentence ``the robot should open the curtains, UNLESS the user is `undressed' in which case the robot does not open the curtains and tells the user `the curtains cannot be opened while you, the user, are undressed,'\phantom{}'' is logically equivalent to the sentence
``IF the user is not `undressed' THEN the robot should open the curtains, AND, IF the user is `undressed' THEN the robot does not open the curtains AND tells the user [...]''.
\end{example}

After what we said about the commutativity of $\UNLESS$ (\Cref{fact:unless}.\ref{item:commutativity}) it would be a \emph{mistake} to think that 
$\phi \UNLESS \psi \INWHICHCASE \chi$ is equivalent to
$\psi \UNLESS \phi \INWHICHCASE \chi$. Hence, the order of occurrence of the different conditions and exceptions in the SLEEC rule \emph{do} have an importance.
%


\section{Identifying the Pieces of Information in SLEEC Rules}
\label{sec:pieces}

In order to formalise a SLEEC rule, one needs to attach a logical term or variable to every \emph{piece of information}. A piece of information in a SLEEC rule is an atomic statement that occurs in the rule. For instance, in \Cref{ex:running}, `the user tells the robot to open the curtains' and `the robot does not open the curtains' are two pieces of information. We identify all the pieces of information of \Cref{ex:running} in \Cref{sec:piecesofinformation}

On close inspection, we can see that the statements in SLEEC rules can be partitioned into two kinds: statements whose truth value is established by sensing the world (e.g., the user tells the robot to open the curtains, or the user is undressed), and statements whose truth value is established by 
the existence of an obligation for the robot to act (e.g., the robot should open the curtains, or the robot tells the user something).

Some further considerations should be taken into account. A logical term must be attached to only one piece of information. Due to the ambiguity of natural languages, this may require a preliminary analysis to disambiguate some statements in the rules. We illustrate it in \Cref{sec:disambiguation}.


\subsection{Disambiguation of Other Words: Tell and Tell}\label{sec:disambiguation}

The rule of the running example contains two occurrences of the word `tells'. 
\begin{quote}
[...] the user \emph{tells} the robot to open the curtains [...] the robot [...] \emph{tells} the user `the curtains cannot be opened [...]'.
\end{quote}
They must be interpreted in two different ways. 
\begin{quote}
[...] the user \emph{asks} the robot to open the curtains [...] the robot [...] \emph{says [to]} the user `the curtains cannot be opened [...]'.
\end{quote}
It is up to the modeller to avoid this sort of ambiguity. When they remain, it is up to the engineer in charge of formalising the rules to detect them.


\subsection{Determining the Pieces of Information in the Domain}\label{sec:piecesofinformation}

Once the statements in the SLEEC rules are disambiguated, we are ready to establish a one-to-one mapping between the set of atomic statements in the rules and a set of logical variables.

In \Cref{ex:running}, we can informally partition the pieces of information into two sets: the variables whose value is known from \emph{sensing} ($a$, $d$, $h$), 
\begin{itemize}
\item [$a$] user asks the robot to open the curtains,
\item [$d$] user is dressed,
\item [$h$] user is highly distressed,
\end{itemize}
and the variables whose value is determined by the existence of an \emph{obligation to act} ($o$, $n$, $t$),
\begin{itemize}
\item [$o$] obligation to open the curtains,
\item [$n$] obligation not to open the curtains,
\item [$s$] obligation to says to the user `the curtains cannot be opened [...]'.
\end{itemize}
This partitioning is helpful to make sense of the rules, but it is inconsequential for our solution.


\section{Compilation of SLEEC Rules Into Logic}
\label{sec:compilation}

We present a general pattern that SLEEC rules follow, and using the observations of \Cref{sec:unlessinwhichcase}, we explain how they must be translated into classical logic.

\subsection{General Pattern for SLEEC Rules}\label{sec:pattern}

\begin{definition}\label{def:sleecpattern}
    A (general pattern of a) SLEEC rule is an expression of the form
\begin{quote}
\noindent IF $C_0$ THEN $O_0$,\\
UNLESS $C_1$ IN WHICH CASE $O_1$,\\
UNLESS $C_2$ IN WHICH CASE $O_2$,\\
\ldots\\
UNLESS $C_n$ IN WHICH CASE $O_n$.
\end{quote}

\medskip

\noindent where $C_i$ and $O_i$, $0 \leq i \leq n$, are arbitrary propositional formulas over a set of propositional variables $AP$.
\end{definition}
\emph{Typically}, in \Cref{def:sleecpattern} $C_i$ is a Boolean condition about sensed variables, and $O_i$ about obligations to act.

\begin{example}
\new{After substituting disambiguated words,}
we obtain the following intermediate forms for \Cref{ex:running}.
\begin{quote}\centering\small
IF the user \emph{asks} the robot to open the curtains [$a$] THEN the robot should open the curtains [$o$], UNLESS the user is `undressed' [$\lnot d$] IN WHICH CASE the robot does not open the curtains [$n$] AND \emph{says to} the user `the curtains cannot be opened while you, the user, are undressed,' [$s$] UNLESS the user is `highly distressed' [$h$] IN WHICH CASE the robot opens the curtains [$o$].
\end{quote}
\[\Downarrow\]
\begin{quote}\centering
\noindent IF $a$ THEN $o$,\\
UNLESS $\lnot d$ IN WHICH CASE $n\land s$,\\
UNLESS $h$ IN WHICH CASE $o$.
\end{quote}
The latter follows the pattern identified in \Cref{def:sleecpattern}.
\end{example}


\subsection{Translation Into Classical Logic}
Prior to presenting the propositional logic form of SLEEC rules, we provide some insights into its derivation.
The general pattern of SLEEC rules can be expressed in a parenthesized manner as follows:\footnote{\new{A possible alternative is IF $C_0$ THEN (\ldots (($O_0$,
UNLESS $C_1$ IN WHICH CASE $O_1$),
UNLESS $C_2$ IN WHICH CASE $O_2$),
\ldots
UNLESS $C_n$ IN WHICH CASE $O_n$), for which a similar analysis can be done.}
}

\begin{quote}
\noindent IF $C_0$ THEN ($O_0$,\\
UNLESS $C_1$ IN WHICH CASE ($O_1$,\\
UNLESS $C_2$ IN WHICH CASE ($O_2$,\\
\ldots\\
UNLESS $C_n$ IN WHICH CASE ($O_n$) \ldots ))).
\end{quote}
Applying \Cref{def:unlessinwhichcase} $n$ times, and 
\emph{uncurrying} (that is, applying the  left-to-right implication of the following logical equivalence $(p \rightarrow (q \rightarrow r)) \equiv (p \land q \rightarrow r)$), we get the following simplified representation:
\begin{quote}
\noindent IF $C_0$ AND NOT $C_1$ THEN $O_0$,\\ 
IF $C_0$ AND $C_1$ AND NOT $C_2$ THEN $O_1$,\\ 
IF $C_0$ AND $C_1$ AND $C_2$ AND NOT $C_3$ THEN $O_2$,\\
\ldots\\
IF $C_0$ AND $C_1$ AND $C_2$ AND \ldots AND $C_{n-1}$ AND NOT $C_n$ THEN $O_{n-1}$,\\
IF $C_0$ AND $C_1$ AND $C_2$ AND \ldots AND $C_{n-1}$ AND $C_n$ THEN $O_{n}$.\\
\end{quote}
Finally, we obtain the following definition of the compilation of a SLEEC rule into classical propositional logic.
\begin{definition}\label{def:compile}
Let $AP$ be a set of propositional variables.
Let $C_i$ and $O_i$, $0 \leq i \leq n$, be arbitrary propositional formulas over $AP$. Let
    \begin{quote}\centering
\noindent $\sigma$ = IF $C_0$ THEN $O_0$,\\
UNLESS $C_1$ IN WHICH CASE $O_1$,\\
UNLESS $C_2$ IN WHICH CASE $O_2$,\\
\ldots\\
UNLESS $C_n$ IN WHICH CASE $O_n$.
\end{quote}
be a SLEEC rule.
We define the \emph{compilation} of $\sigma$, noted $\compile(\sigma)$, as follows:
\[
\left [ \bigwedge_{0 \leq i \leq n-1} 
\left ( \left ( \bigwedge_{0 \leq j \leq i} ( C_j  ) \land \lnot C_{i+1} \right ) \rightarrow  O_i \right ) \right ]
\]
\[
\land \left [ \left ( \bigwedge_{0 \leq j \leq n} ( C_j ) \right ) \rightarrow O_n \right ] \enspace .
\]
\end{definition}
We are finally ready to give the last step in the compilation of the SLEEC rule of \Cref{ex:running}.
\begin{example}\label{ex:compilation}
The SLEEC rule of \Cref{ex:running} is compiled as follows.
\begin{quote}\centering
\noindent IF $a$ THEN $o$,\\
UNLESS $\lnot d$ IN WHICH CASE $n\land s$,\\
UNLESS $h$ IN WHICH CASE $o$\enspace .
\end{quote}
\[\Downarrow \compile\]
    \[
    (a \land \lnot\lnot d \rightarrow o) \land (a \land  \lnot d \land  \lnot h \rightarrow n \land  s) \land  (a \land  \lnot d \land h \rightarrow o) \enspace .
    \]
    
\end{example}
The semantics of SLEEC rules follows.
\begin{fact}\label{fact:semantics-sleec-rules}
Let $\sigma$ be a SLEEC rule as in \Cref{def:sleecpattern}, and let $v$ be an interpretation. We have: $v(\sigma) = true \text{ iff }$
\begin{align*}
     & \big[\forall 0 \leq i \leq n-1, (\exists 0 \leq j \leq i, v(C_i) = false\\ 
    & \text{ or } v(C_{i+1}) = true \text{ or } v(O_i) = true)\big] \text{ and }\\ 
    & \big [ \exists 0 \leq i \leq n, v(C_i) = false \text{ or } v(O_n) = true \big ] \enspace.
\end{align*}
\end{fact}
The size of the compiled form of a SLEEC rule is quadratic in the size of the SLEEC rule.
\begin{fact}\label{prop:size-compilation}
$|\compile(\sigma)| = O(|\sigma|^2)$.
\end{fact}
This is instrumental to establish the complexity results of \Cref{sec:complexity}.

\section{Complexity of Reasoning About SLEEC Rules}
\label{sec:complexity}

Since we are dealing with a special form of propositional formulas, we provide a few observations on the theoretical complexity of reasoning about SLEEC rules. 
\begin{proposition}\label{prop:complexity}
    Let $AP$ be a set of propositional variables.
    Deciding whether an obligation $\phi$ (more generally, any propositional statement over $AP$) is entailed by a set $\Gamma$ of SLEEC rules is \coNP-complete, even
    if the terms $C_0, \ldots C_n$ and $O_0, \ldots O_n$ in every rule are restricted to be literals in $\{p, \lnot p \mid p \in AP\}$.
\end{proposition}
\begin{proof}
For membership, consider the complement problem, asking whether the set of rules $\Gamma$ does not entail the statement $\phi$.
This can be solved by the following algorithm:
compile the SLEEC rules following \Cref{def:compile} to obtain a conjunction of propositional formulas with only a quadratic blowup in size (\Cref{prop:size-compilation}); non-deterministically choose a valuation; check whether it satisfies all rules in $\Gamma$ and does not satisfy $\phi$. It is correct and sound, and runs in polynomial time on a non-deterministic Turing machine.
This means that our problem is in \coNP.

For hardness, we can polynomially reduce 3SAT.
For $p \in AP$, define $\sim p = \lnot p$ and $\sim \lnot p = p$.
An instance of 3SAT is a formula of the form:
\[
(l_{1,1} \lor l_{1, 2} \lor l_{1, 3}) \land (l_{2,1} \lor l_{2, 2} \lor l_{2, 3})\ldots \land (l_{k,1} \lor l_{k, 2} \lor l_{k, 3}) \enspace ,
\]
with every $l_{i,j} \in \{p, \lnot p \mid p \in AP\}$.
Each clause $(l_{i,1} \lor l_{i,2} 
 \lor l_{i,3})$ can be captured by the SLEEC rule corresponding to
\begin{quote}
IF $\sim l_{i,1}$ THEN $l_{i,3}$,\\
UNLESS $l_{i,2}$ IN WHICH CASE $\top$
\end{quote}
Indeed, its logical form is
$(\sim l_{i,1} \land \sim l_{i,2} \rightarrow l_{i,3}) \land (\sim l_{i,1} \land l_{i,2} \rightarrow \top)$, which is equivalent to $l_{1,1} \lor l_{1, 2} \lor l_{1, 3}$.
It means that deciding an entailment problem from an arbitrary 3CNF over $AP$ with $k$ clauses can be reduced to deciding an entailment problem from $k$ SLEEC rules where every term is a literal. Entailment for 3CNF (3SAT) is \coNP-hard, so hardness follows. \qed
\end{proof}
A consequence of \Cref{prop:complexity} is that reasoning with the rules of~\cite{Townsendetal22mindsandmachines} is (worst-case) hard unless one imposes drastic restrictions.
An example of restriction that makes the entailment of an obligation from a set of SLEEC rules easier is when all $C_i$ are negative literals and all $O_i$ are positive literals.
\begin{proposition}\label{prop:horn-ptime}
    Let $AP$ be a set of propositional variables.
    Deciding whether an obligation $o$ is entailed by a set $\Gamma$ of SLEEC rules is in \PTIME when the terms $C_0, \ldots C_n$ of all rules are all negative literals in $\{\lnot p \mid p \in AP\}$ and all $O_0, \ldots O_n$ of all rules are positive literals in $\{p \mid p \in AP\}$.
\end{proposition}
It follows from a simple reduction to HORNSAT which is in \PTIME~\cite{DOWLING1984267} (\new{see appendix 
for details}).

Another consequence of \Cref{prop:complexity} is that there is no reason to restrict ourselves to $C_0, \ldots C_n$ and $O_0, \ldots O_n$ being only literals. 
We can use complex $C_i$ with no impact on the  computational complexity. Since $C_i$ describe states of affairs, it is very convenient 
to be able to write them as arbitrary propositional formulas. We can do so without impacting the complexity of reasoning.

We can use complex $O_i$.
For instance, if two obligations $o_1$ and $o_2$ should be triggered at the same time, it would not alter the complexity of reasoning to simply write $o_1 \land o_2$ for some $O_i$. In the rule of \Cref{ex:running}, we have: ``in which
case the robot does not open the curtains and tells the user `the curtains cannot be opened while you, the user, are undressed'\phantom{}''. We could use a single atom $o_{n,s}$ to represent the obligation to not open the curtains and the obligation to say something to the user, but it does not hurt the computational complexity of reasoning if we use two obligations $n$ and $s$, and simply identify a $O_i$ with $n \land s$. 
We will make good use of this.

\section{Reasoning With SLEEC Rules}
\label{sec:reasoning}

We build upon the example provided in \Cref{ex:running} to demonstrate the application of the compiled logical form of a SLEEC rule for automated normative reasoning. We give an overview of the reasoning tasks in propositional logic,  Answer Set Programming, and in PROLOG.

In ASP and PROLOG (\Cref{sec:lp-asp} and \Cref{sec:lp-prolog}), we use the following predicates:
\begin{itemize}
    \item \texttt{a(X,Y)}: user \texttt{X} \textbf{a}sks to open \texttt{Y}.
    \item \texttt{s(X,Y)}: obligation to \textbf{s}ay \texttt{Y} to \texttt{X}.
    \item \texttt{o(X)}: obligation to \textbf{o}pen \texttt{X}.
    \item \texttt{n(X)}: obligation to \textbf{n}ot open \texttt{X}.
    \item \texttt{d(X)}: \texttt{X} is \textbf{d}ressed.
    \item \texttt{h(X)}: \texttt{X} is \textbf{h}ighly distressed.
\end{itemize}
We start with reasoning in propositional logic.

\subsection{Propositional Reasoning}

The SLEEC rule of \Cref{ex:running} is compiled into the following propositional formula:
$(a \land \lnot\lnot d \rightarrow o) \land (a \land  \lnot d \land  \lnot h \rightarrow n \land  s) \land  (a \land  \lnot d \land h \rightarrow o)$. Let us note it $\phi_{SLEEC}$. Any off-the-shelf SAT solver can now be used for reasoning.


Suppose that user asks, user is dressed, user is highly distressed. Together with the SLEEC rule, it \emph{entails} the obligation to open the curtains. This is captured by:
    \[
    {(a \land d \land h)} ~~~ \land ~~~ \phi_{SLEEC} \models o \enspace .
    \]
Now suppose that user asks, user is not dressed, user is not highly distressed. Then, together with SLEEC rule, it \emph{entails} the obligation to not open the curtains and say why:
    \[
    {(a \land \lnot d \land \lnot h)} ~~~ \land ~~~\phi_{SLEEC} \models n \land s \enspace .
    \]
To ensure the soundness of our solution, it is imperative that it does not propose any spurious obligations. Fortunately, our solution meets this requirement, as no additional obligations are entailed by the previous formulas. For instance, $(a \land d \land h) \land \phi_{SLEEC} \not \models n \lor s$, because the interpretation $v = \{a \mapsto true, d \mapsto true, h \mapsto true, \bm{n \mapsto false}, o \mapsto true, \bm{s \mapsto false} \}$ is a model of $(a \land d \land h) \land \phi_{SLEEC}$. Similarly, we have that $(a \land \lnot d \land \lnot h) \land \phi_{SLEEC} \not \models o$,
because the interpretation $v' = \{a \mapsto true, d \mapsto false, h \mapsto false, n \mapsto true, \bm{o \mapsto false}, s \mapsto true \}$ is a model of $(a \land \lnot d \land \lnot h) \land \phi_{SLEEC}$.


\subsection{Logic Programming: Answer Set Programming}
\label{sec:lp-asp}

\Cref{ex:running} can be formalised in ASP as follows.
\begin{scriptsize}
\begin{verbatim}
o(Y) :- a(X,Y), d(X).
o(Y) :- a(X,Y), not d(X), h(X).
n(Y) :- a(X,Y), not d(X), not h(X).
s(X, Y) :- a(X,Y), not d(X), not h(X).
a(user,curtains).
d(someoneelse).
h(user).
\end{verbatim}
\end{scriptsize}
We have chosen to represent a situation where the user is not dressed (\texttt{d(user)} is not present) but highly distressed (\texttt{h(user)}).
The clause \texttt{d(someoneelse)} is present for technical reason, because all predicates must appear in the head of a rule.
We obtain the expected model in the ASP system clingo (5.7.0):
\begin{scriptsize}
\begin{verbatim}
Solving...
Answer: 1
d(someoneelse) a(user,curtains) 
o(curtains) h(user)
SATISFIABLE
\end{verbatim}
\end{scriptsize}
That is, if the user asks the robot to open the curtains, and the user is highly distressed, then the only model is when the robot has the obligation to open the curtains (\texttt{o(curtains)}).

\subsection{Logic Programming: PROLOG}
\label{sec:lp-prolog}

A straightforward implementation of \Cref{ex:running} in PROLOG is 
as follows. 
\begin{scriptsize}
\begin{verbatim}
o(Y) :- a(X,Y), d(X),
        write("I have the obligation to open "), 
        write(Y).
o(Y) :- a(X,Y), \+ d(X), h(X),
        write("I have the obligation to open "), 
        write(Y).
n(Y) :- a(X,Y), \+ d(X), \+ h(X),
        write("I have the obligation not to open "), 
        write(Y).
s(X, Y) :- a(X,Y), \+ d(X), \+ h(X),
        write("I can't open "), write(Y), 
        write(" while you "), write(X), 
        write(" are undressed").
a(user,curtains).  % user asks robot open the curtains
%d(someoneelse).
d(user).           % user is dressed
h(someoneelse).
%h(user).          % user is highly distressed
\end{verbatim}
\end{scriptsize}
We simply enriched the rules with \texttt{write} predicates to provide more informal details in the query answers. Clauses with \texttt{someoneelse} are used to ensure that all predicates appear in the head of a rule.

To get all the obligations, we can ask the PROLOG query:
\begin{scriptsize}
\begin{verbatim}
o(Z) ; n(Z) ; s(X, Y).
\end{verbatim}
\end{scriptsize}
In SWI-Prolog (9.1.20), it yields
\begin{scriptsize}
\begin{verbatim}
I have the obligation to open curtains
Z = curtains
\end{verbatim}
\end{scriptsize}
Similar to the machine-processable input program, the output of the reasoning tasks is also machine processable. Notably, the conclusions drawn from the reasoning can be seamlessly transferred to other components of a robot platform, such as a planning module. 


Alternatively, if the user is not dressed (that is if the clause \texttt{d(user)} is not present) then the same query is answered with:
\begin{scriptsize}
\begin{verbatim}
I have the obligation not to open curtains
Z = curtains
I can't open curtains while you user are undressed
X = user,
Y = curtains    
\end{verbatim}
\end{scriptsize}
If the user is not dressed but highly distressed (that is if the clause \texttt{d(user).} is not present, but the clause \texttt{h(user).} is present), 
then the same query is answered with:
\begin{scriptsize}
\begin{verbatim}
I have the obligation to open curtains
Z = curtains
\end{verbatim}
\end{scriptsize}


\section{Conclusions}
\label{sec:conclusion}

We have bridged the crucial divide that separates the norms articulated as SLEEC rules by philosophers, ethicists, lawyers, domain experts, and other professionals from the realization of agents compliant with these norms on robot development platforms.

In this study, we have conducted a linguistic and logical analysis of the social, legal, ethical, empathetic, and cultural (SLEEC) rules obtained through the elicitation method outlined by \citet{Townsendetal22mindsandmachines}.
We have established a precise logical representation of SLEEC rules and investigated their computational characteristics. Furthermore, we have demonstrated how this definition can be effortlessly employed for SLEEC-compliant automated reasoning by leveraging SAT solvers or logic programming frameworks, all of which are conveniently readily-available on robot development platforms.




%% file: supplementarymaterial.tex
\clearpage

\section{Proof of \Cref{prop:horn-ptime}}

A \emph{Horn formula} is a conjunction of clauses containing at most one positive literal. A \emph{dual-Horn formula} is a conjunction of clauses containing at most one negative literal.

\begin{nothing}[\Cref{prop:horn-ptime}]
    Let $AP$ be a set of propositional variables.
    Deciding whether an obligation $o$ is entailed by a set $\Gamma$ of SLEEC rules is in \PTIME when the terms $C_0, \ldots C_n$ of all rules are all negative literals in $\{\lnot p \mid p \in AP\}$ and all $O_0, \ldots O_n$ of all rules are positive literals in $\{p \mid p \in AP\}$.
\end{nothing}
\begin{proof}
    Every conjunction in every SLEEC rule compiled according to \Cref{def:compile} is of the form $C_0 \land \ldots C_i \land \lnot C_{i+1} \rightarrow O_i$ or $C_0 \land \ldots C_i \land C_{n} \rightarrow O_n$.

    The first form is equivalent to $\lnot C_0 \lor \ldots \lnot  C_i \lor C_{i+1} \lor O_i$. Since $C_0, \ldots C_{i+1}$ are negative literals and $O_i$ is a positive literal, it is a dual-Horn clause, with $C_{i+1}$ being the only negative literal ($\lnot C_0, \ldots, \lnot  C_i$ are positive literals).

    The second form is also equivalent to a dual-Horn clause (with only positive literals).

    $\Gamma$ entails $o$ iff the set of dual-Horn clauses composed of all the clauses in $\Gamma$ and the dual-Horn clause $\lnot o$ is 
    not satisfiable. We obtain an equi-satisfiable set of Horn clauses by simply negating all variables. Using \Cref{prop:size-compilation}, we know that this reduction can be done in polynomial time.
    Hence, we know that the problem can be decided in polynomial time because HORNSAT is in \PTIME~\cite{DOWLING1984267}. \qed
\end{proof}

\section{Another Scenario: Post Office}
\label{sec:post-office-scenario}

This scenario makes use of  more intricate relations among the objects of the domain. In fact, this scenario is more aptly represented with a predicative language like PROLOG, than with plain propositional logic.

\medskip

An administrative service requires to come to a public desk, and requires to check-in at the time of arrival.
There is an integer counter. At check-in, the user is assigned a number---the value of the counter---and the counter is then incremented.
The default SLEEC rule for the on-site scheduling is simply as follows.
\begin{quote}
If turn $t$, call user assigned to number $t$.
\end{quote}
A user can specify preferences or soft ethics~\cite{floridi18}. 
E.g., a user can prefer to swap her turn with another user who checked in later, if this user has some kind of handicap; and a user with a certain handicap can prefer to swap her turn with another user who checked in earlier. The on-site scheduling automaton is capable of recognising the handicaps of the users, and is capable of accessing their soft ethics. It must then take all this information into account when calling the next user. The SLEEC rule becomes as follows.
\begin{quote}
    If turn $t$, call user assigned to number $t$, UNLESS there is a user with handicap $H$, and assigned to a number $t'$, and user with number $t$ opted in giving priority to users with handicap $H$, and user with number $t'$ with handicap $H$ opted in taking the priority from another user, and $t' > t$, IN WHICH CASE call user with number $t'$ AND thank user with number $t$.
\end{quote}
The pieces of information that correspond to sensing are
\begin{itemize}
    \item the turn number,
    \item the user assigned to a certain turn number,
    \item the handicaps of a user,
    \item the soft ethics of a user with regards to giving the priority to users with a certain handicap,
    \item the soft ethics of a user with regards to accepting the priority for their handicap,
\end{itemize}
while those corresponding to obligations to act are
\begin{itemize}
    \item obligation to call a user with a certain turn number,
    \item obligation to thank a user.
\end{itemize}
We use a defined predicate \texttt{can\_priority\_swap} for brevity. It simply captures the statement between `UNLESS' and `IN WHICH CASE' in the raw SLEEC rule. 
That is to say, the rule simply becomes:
\begin{quote}
    If turn $t$, call user assigned to number $t$, UNLESS `user with number $t$ can swap with user with number $t'$' IN WHICH CASE call user with number $t'$ AND thank user with number $t$,
\end{quote}
with the appropriate definition of \texttt{can\_priority\_swap}.

The following is the result of the compilation. 
\begin{scriptsize}
\begin{verbatim}
% obligations
call_number(X) :- turn(X), \+ can_priority_swap(X,_,_,_,_),
    write("I'm calling "), has_number(P,X), write(P).
call_number(X) :- turn(Y), can_priority_swap(Y,X,_,P,_),
    write("I'm calling "), write(P).
thank(P) :- turn(Y), can_priority_swap(Y,_,P,_,_),
    write("Thank you "), write(P).

% helpers
can_priority_swap(X,Y,P1,P2,H) :- has_number(P1, X), 
    has_number(P2, Y),
    gives_priority(P1, H), asks_priority(P2, H), X < Y, !.
obligation(Z) :- call_number(Z) ; thank(Z).

% data
turn(3).
has_number(ana,3).
has_number(bob,4).
has_number(luc,5).
gives_priority(ana,blind).
asks_priority(luc,blind).
\end{verbatim}
\end{scriptsize}
The data sets the facts that it is turn $3$, \texttt{ana} has turn $3$, \texttt{bob} has turn $4$, \texttt{luc} has turn $5$, \texttt{ana} accepts to give priority to blind persons, and \texttt{luc} accepts the priority if given to him for being blind.
The predicate \texttt{obligation} is defined only for convenience. It can be used directly to query all the obligations \texttt{call\_number} or \texttt{thank} as follows:
\begin{scriptsize}
\begin{verbatim}
obligation(X).
\end{verbatim}
\end{scriptsize}
We get the answer:
\begin{scriptsize}
\begin{verbatim}
I'm calling luc
Z = 5
Thank you ana
Z = ana
\end{verbatim}
\end{scriptsize}
If instead, for instance \texttt{luc} were not asking the priority (that is, if \texttt{asks\_priority(luc,blind)} were not present), the same query would return the answer:
\begin{scriptsize}
\begin{verbatim}
I'm calling ana
Z = 3    
\end{verbatim}
\end{scriptsize}